\documentclass[sigconf]{aamas} 

\usepackage{macros}
\usepackage{tcolorbox} 
\usepackage{algorithm} 
\usepackage[noend]{algpseudocode}
\usepackage{balance} 

\setcopyright{ifaamas}
\acmConference[AAMAS '26]{Proc.\@ of the 25th International Conference
on Autonomous Agents and Multiagent Systems (AAMAS 2026)}{May 25 -- 29, 2026}
{Paphos, Cyprus}{C.~Amato, L.~Dennis, V.~Mascardi, J.~Thangarajah (eds.)}
\copyrightyear{2026}
\acmYear{2026}
\acmDOI{}
\acmPrice{}
\acmISBN{}

\acmSubmissionID{494}

\title{Health Facility Location in Ethiopia: Leveraging LLMs to Integrate Expert Knowledge into Algorithmic Planning }

\author{Yohai Trabelsi}
\authornote{Work conducted while the author was affiliated with Harvard University, USA.}
\affiliation{
  \institution{Department of Data Science and AI, School of Computer Science,\\
  Ariel University}
  \city{Ariel}
  \country{Israel}}
\email{yohai.trabelsi@gmail.com}

\author{Guojun Xiong}
\authornote{Correspondence to:  Guojun Xiong (xionggj1@gmail.com).}
\affiliation{
  \institution{John A. Paulson School of Engineering and Applied Sciences,
  Harvard University}
  \city{Cambridge, MA}
  \country{United States}}
\email{xionggj1@gmail.com}

\author{Fentabil Getnet}
\affiliation{
  \institution{National Data Management and Analytics Center for Health, Ethiopian Public Health Institute}
  \city{Addis Ababa}
  \country{Ethiopia}}
  \email{b.infen4ever@gmail.com}

\author{St\'{e}phane Verguet}
\affiliation{
  \institution{Department of Global Health and Population, Harvard T.H. Chan School of Public Health}
  \city{Boston, MA}
  \country{United States}}
\email{verguet@hsph.harvard.edu}

\author{Milind Tambe}
\affiliation{
  \institution{John A. Paulson School of Engineering and Applied Sciences,
  Harvard University}
  \city{Cambridge, MA}
  \country{United States}}
\email{milind_tambe@harvard.edu}

\begin{abstract}

Ethiopia’s Ministry of Health is upgrading health posts to improve access to essential services, particularly in rural areas. Limited resources, however, require careful prioritization of which facilities to upgrade to maximize population coverage while accounting for diverse expert and stakeholder preferences.
In collaboration with the Ethiopian Public Health Institute and Ministry of Health, we propose a hybrid framework that systematically integrates expert knowledge with optimization techniques.
Classical optimization methods provide theoretical guarantees but require explicit, quantitative objectives, whereas stakeholder criteria are often articulated in natural language and difficult to formalize.
To bridge these domains, we develop the Large language model and Extended Greedy (LEG) framework. Our framework combines a provable approximation algorithm for population coverage optimization with LLM-driven iterative refinement that incorporates human-AI alignment to ensure solutions reflect expert qualitative guidance while preserving coverage guarantees. Experiments on real-world data from three Ethiopian regions demonstrate the framework’s effectiveness and its potential to inform equitable, data-driven health system planning.

\end{abstract}

\keywords{Health Facility Location, Optimization, Human Expert Knowledge, Human-AI Alignment, LLM}

\newtheorem{remark}{Remark}
         
\newcommand{\BibTeX}{\rm B\kern-.05em{\sc i\kern-.025em b}\kern-.08em\TeX}

\begin{document}


\pagestyle{fancy}
\fancyhead{}


\maketitle 

\section{Introduction}
Ensuring equitable access to essential health services remains a central challenge for Ethiopia’s Ministry of Health (MOH). Since the launch of the Health Extension Program (HEP) in 2003/04 \cite{wang2016ethiopia},  the country has made significant progress in expanding basic care delivery to rural populations. The HEP Optimization Roadmap (2020–2035)\cite{MoH_Ethiopia_2020_roadmap} further envisions upgrading selected health posts into comprehensive ones that can provide advanced 
services such as childbirth and postnatal care (see Figure~\ref{fig:hp}). However, upgrading facilities is expensive, and the available public budget is severely constrained ~\cite{choo2026optimizing}. Determining which facilities to upgrade thus becomes a complex optimization problem involving limited resources, heterogeneous population needs, and diverse stakeholder opinions.

\begin{figure}
    \centering
    \includegraphics[width=0.45\linewidth]{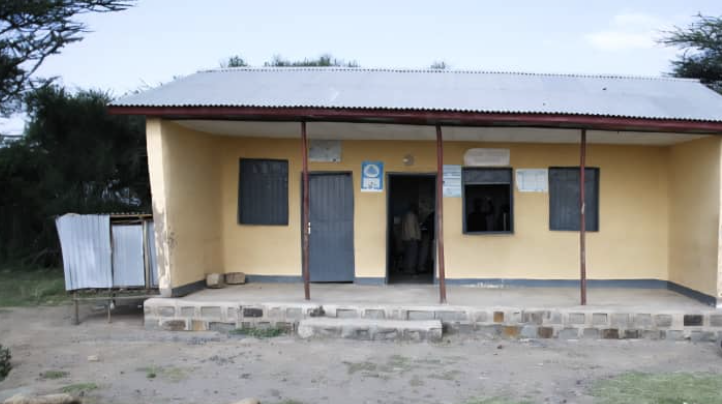}
    \includegraphics[width=0.49\linewidth]{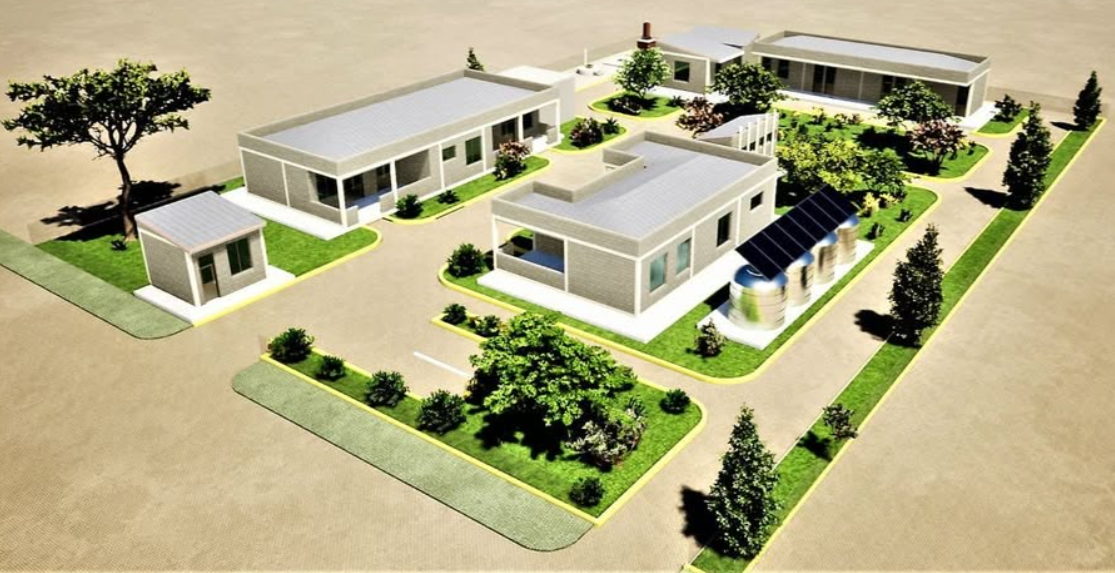}
        \caption{Basic health post (left); Design of a comprehensive health post that provides more essential services (right). Source: Ministry of Health, Ethiopia.} 
    \Description{Left: basic health post; right: comprehensive health post offering more essential services. HEP MOH technical report.}        
    \label{fig:hp}
    
\end{figure}

\begin{figure}[htb]
     \centering
    \includegraphics[width=0.49\linewidth]{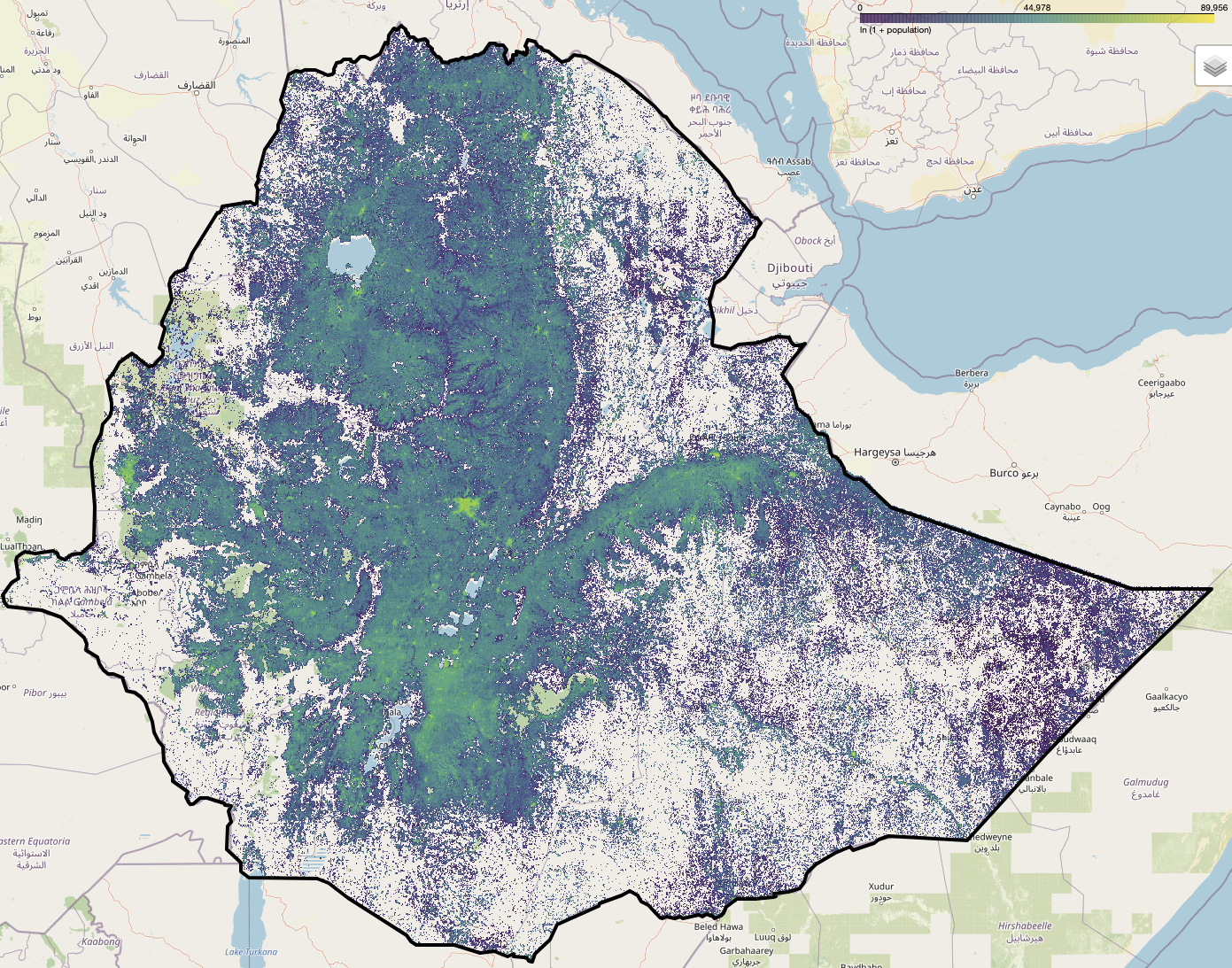}
     \includegraphics[width=0.49\linewidth]{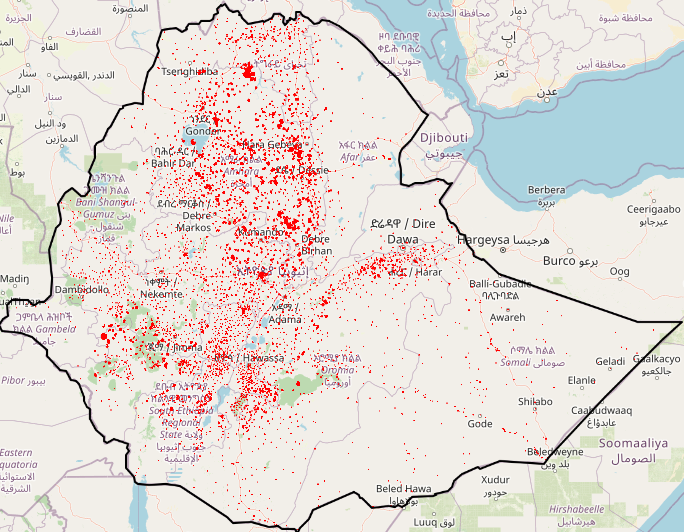}
     \caption{(\cite{choo2026optimizing}) Map of Ethiopia overlaid with projected population estimates for 2026 (log scale). Brighter yellow areas on the left indicate higher population densities. The map on the right shows the locations of health facilities capable of providing essential health services (red points). Many populations currently lack 2-hour access to such facilities \cite{hendrix2023estimated}.}

     \Description{Map of Ethiopia overlaid with projected 2026 population estimates (log scale).}
     \label{fig:intro-figure}
\end{figure}

Recent studies have applied optimization and geospatial methods to identify locations where comprehensive facilities are needed ~\cite{choo2026optimizing,alemu2025optimizing,haileamlak2023ethiopian}. These efforts primarily focus on maximizing population coverage under distance and capacity constraints (See Figure~\ref{fig:intro-figure}). Yet, in practice, the final allocation decisions in the stepwise construction of comprehensive health posts remain dominated by expert judgment and stakeholder negotiation, rather than by algorithmic outputs.
While human expertise captures important contextual knowledge—such as terrain accessibility or local socio-political considerations—it is often expressed in natural language and difficult to encode into a mathematical objective function \cite{xiong2025vortex,christiano2017deep}. 
It is possible that even the advisors themselves may struggle to provide their preferences as a fully coherent set of recommendations.

Recent advances in Large Language Models (LLMs)~\cite{xiong2025vortex,behari2024decision} offer a promising avenue for bridging the gap between qualitative human judgment and quantitative optimization.
LLMs can interpret and structure unformalized expert advice, enabling algorithms to incorporate contextual and domain-specific reasoning that traditional models often overlook.
Nevertheless, these methods typically lack formal theoretical guarantees on performance or stability—an issue that is especially critical in high-stakes domains such as public health or infrastructure planning.
In the absence of such guarantees, language-based systems risk producing allocations that appear reasonable linguistically but fail to satisfy fairness, transparency, or policy-alignment criteria required for real-world adoption.

To address this gap, we develop a hybrid framework that couples algorithmic optimization with language-based expert reasoning. Our approach builds upon classical submodular maximization, ensuring provable guarantees on coverage, while leveraging LLMs to interpret and iteratively integrate human advice expressed in text. The LLM acts as a bridge between formal optimization and informal domain knowledge—translating verbal recommendations into structured allocation adjustments that preserve theoretical performance bounds. This integration allows the system to remain both rigorous and human-aligned.

We demonstrate the framework in collaboration with the Ethiopian Public Health Institute and Ministry of Health, focusing on three regions. Empirical results show that LLM-guided iterations significantly improve the alignment between algorithmic allocations and expert advice, while maintaining high coverage efficiency. Beyond this specific application, the proposed method provides a general blueprint for embedding qualitative human preferences into quantitative decision-making processes.

\subsection{Process description}
\label{sec:exmpl}
In Figure~\ref{fig:scheme} we present the proposed method in a high level.
The process begins with the problem inputs, an initial greedy allocation, and a list of advice sentences. These inputs are provided to a large language model (e.g., Gemini, ChatGPT), which iteratively refine the allocation strategy. The model proposes an allocation that is optimized using a guided greedy procedure parameterized by $(\alpha, \beta)$. 
The resulting outcome is then evaluated, and the evaluation feedback is transformed into a new prompt for further refinement. This iterative loop continues until the maximum number of iterations is reached, producing the final optimized allocation.
\begin{figure}
    \centering
        \includegraphics[width=0.9\linewidth]{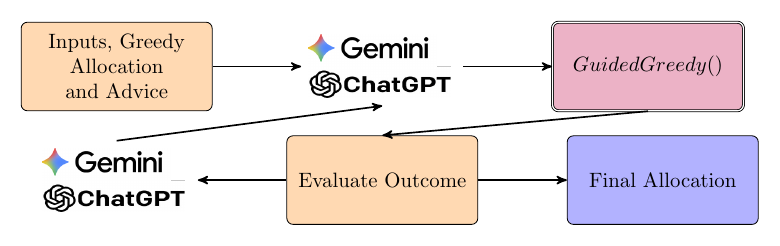}
        \caption{Overview of the proposed method.}
        \Description{Flow diagram with labeled boxes connected by arrows, showing inputs, processing steps, evaluation, and final allocation output.}
    \label{fig:scheme}
\end{figure}

\subsection{Contributions}
In collaboration with the Ethiopian Public Health Institute and Ministry of Health, this work introduces the unified LEG framework for upgrading health facilities in Ethiopia that jointly optimizes population coverage and alignment with expert guidance.
Our main contributions are summarized as follows:

$\triangleright$ \emph{LLM-integrated multi-objective optimization.} We propose a novel framework that embeds LLMs into a submodular optimization pipeline, enabling the joint optimization of quantitative coverage objectives and qualitative human advice.

$\triangleright$ \emph{Theory–practice trade-off with controllable guarantees.} We formalize a tunable mechanism to balance provable approximation guarantees and adherence to expert-driven recommendations.

$\triangleright$ \emph{Online budget adaptability.} We extend the framework to handle sequential or dynamically arriving budgets, a critical requirement for multi-year public-sector planning.

$\triangleright$ \emph{Real-world validation in Ethiopia.} We conduct extensive experiments on three regional datasets (Afar, Somali, and Benishangul-Gumuz), demonstrating that our method maintains strong coverage performance while significantly improving interpretability and stakeholder alignment.
 
\section{Related Works}
\subsection{Plan-and-execute}

Many works are discussing the plan-and-execute approach where a high-level plan is computed before the execution of the tasks (e.g., \cite{yang2024agentoccam, zhang2025webpilot,sun2023adaplanner}).
Some of these works \cite{hu2025divide,erdogan2025plan} explore LLM-based hierarchical planning. 
However, their hierarchical design primarily serves to simplify complex problems, whereas our setting is closer to a multi-objective framework- requiring both high-level (abstract) performance and low-level (concrete) performance to be valuable independently.

\subsection{Multi-objective optimization}

The field of multi-objective optimization~(for overviews, see \cite{roijers2013survey,peitz2025multi}) provides a principled foundation for handling competing objectives and evaluating trade-offs among them. A widely adopted strategy is weighted scalarization~\cite{miettinen1999nonlinear,hayes2022practical, kimnavigating}, which collapses multiple objectives into a single composite objective using pre-specified weights. Despite its broad applicability, scalarization presents a long-standing difficulty: effectively traversing the Pareto frontier typically demands manual or iterative tuning of these weights to reveal different trade-off solutions.
To overcome this limitation, recent work such as \cite{xiong2025vortex} proposes a language-driven framework for exploring Pareto-optimal solutions. Instead of explicitly setting scalarization coefficients, an LLM interprets verbal feedback and adjusts the reward structure accordingly, enabling a more natural and user-centric interaction with the optimization process.
In contrast, our approach operates within a hierarchically coordinated framework. Here, the LLM and the solver interact symbiotically—the LLM contributes high-level reasoning and structural insights, while the solver provides robust low-level optimization capabilities. This interplay allows the system to retain theoretical rigor while integrating expert intuition in a seamless and interpretable manner.

\subsection{Using LLMs for reward shaping}
Recent research has increasingly explored the use of LLMs for reward shaping and specification~\cite{ma2024eureka,xiong2025vortex}.
Early systems such as Eureka~\cite{ma2024eureka} demonstrated that LLMs can automatically synthesize reward functions directly from natural-language descriptions of desired behaviors.
Similarly, \cite{kwon2023reward} shows that language interfaces can simplify the design of proxy rewards through interactive refinement.
The Decision-Language Modeling  framework~\cite{behari2024decision,verma2025balancing} extends this paradigm to societal domains, using LLMs to propose reward formulations for public-health resource allocation.
\cite{kim2025preference} proposed a method to fine-tune LLM to efficiently translate natural-language human preferences into reward functions.
Although these approaches highlight the expressive power of language-driven reward design, they often risk distorting the underlying task utility, since optimization is performed purely with respect to human-specified or textual preferences.

To mitigate this issue, \cite{xiong2025vortex} integrates language-derived shaping rewards into an existing solver’s intrinsic objective, maintaining the base utility while progressively improving alignment with user intent.
Yet, the requirement to explicitly construct a numerical reward function for every form of expert advice remains a practical bottleneck—translating nuanced guidance into precise reward formulations can itself demand significant effort and domain expertise.

\subsection{Health facility location}
Many studies have addressed health facility location problems (see, e.g., the survey by \cite{ahmadi2017survey}).
\cite{wang2023research} examined the allocation of emergency health facilities while incorporating expert preferences into their model.
\cite{flores2021optimizing} examined the placement of health facilities in the Philippines, explicitly considering trade-offs between equity and efficiency in their optimization approach.
\cite{shariff2012location} optimized clinic locations in Kuala Langat, Malaysia, to maximize population coverage within 3–5 km.
Other studies, such as \cite{alemu2025optimizing, choo2026optimizing}, focused on optimizing health facility locations in Ethiopia.
However, most of these studies do not account for alignment with human expert preferences.

\subsection{Human-AI alignment}
Recent work has focused on aligning AI decisions with human expertise (see \cite{ji2023ai, shen2024towards} for surveys) and aligning LLMs
to better suit human-oriented tasks and expectations (see \cite{wang2023aligning} for a survey). 
Our work contributes to the literature on aligning algorithmic decision-making with human preferences (e.g., \cite{ng2000algorithms, abbeel2004apprenticeship, xiong2025vortex}). 
However, whereas most existing approaches rely on the explicit construction of a numerical reward function, which is not always practical, our method enables LLMs to directly guide the allocation process, bypassing the need for an intermediate reward-function construction step.

\section{Settings and Problem Formulation}

\subsection{Settings}

Let $r \in \mathbb{N}_{>0}$ denote the number of districts.  
Let the ground set be $V = T_1 \uplus \ldots \uplus T_r$, where each $T_i$ represents the set of candidate grid cells in district $i$, and $\uplus$ denotes a disjoint union (See the example in Section \ref{sec:exmpl}). 
For simplicity, we assume that all grid cells in these districts are candidates for locating a facility.\footnote{This assumption simplifies exposition; our framework remains valid even if only a subset of cells within each district are eligible, as the optimization operates on any predefined candidate set.}  
A \textbf{grid-cell allocation} is any subset $S \subseteq V$.  

For any grid-cell allocation $S \subseteq V$, we define the \emph{district allocation} $h(S) \in \mathbb{N}^r$ by
$$
h_i(S) = |S \cap T_i|, \quad \text{for } i = 1, \dots, r,
$$
so that $h_i(S)$ represents the number of facilities placed in district $i$.  
We denote by $b$ the total budget, i.e., the number of health posts to be upgraded in a given year.

\paragraph{Coverage function.}
Let $\mathrm{covered}(S)$ be the set of grid cells served by the facilities in $S$.  
The total coverage associated with $S$ is defined as
\begin{equation}
\label{eq:ethiopia-objective-f}
f(S) = \sum_{c \in \mathrm{covered}(S)} w_c,
\end{equation}
where $w_c$ denotes the population of grid cell $c$.  
This monotone submodular function quantifies the cumulative population covered by the selected facilities.

\paragraph{Advice alignment function.}
Because the notion of “alignment with expert advice’’ is inherently qualitative, defining a corresponding numerical objective is nontrivial.  
We represent this alignment through a function $g : 2^V \to \mathbb{R}$ that assigns a score to each grid-cell allocation, reflecting how well the allocation adheres to the provided human or LLM-generated advice.  
(Here, $2^V$ denotes the set of all possible subsets of $V$.)

For evaluation, we approximate this alignment function using a set of advice sentences $A$.  
For each advice sentence $a \in A$, we define an auxiliary scoring function $g_a : 2^V \to \mathbb{R}$ that measures the extent to which the advice $a$ is satisfied.  
An LLM is used to define these component functions and aggregate them into an overall alignment metric:
\begin{equation}
\label{eq:ethiopia-objective-g}
g_{\mathrm{eval}}(S) = \sum_{a \in A} g_a\big(h(S)\big).
\end{equation}
This construction serves as a proxy for human evaluation, enabling language-based advice to be incorporated into a quantitative optimization framework. 

\subsection{Problem formulation}

\paragraph{The multi-objective problem.}
Given the functions $f$ and $g$, and a total budget $b$, we define the following multi-objective formulation that jointly optimizes \emph{population coverage} and \emph{alignment with expert advice}:
\begin{equation}
\label{eq:ethiopia-multi-obj}
\max_{S \subseteq V,\, |S| = b} \big\{ f(S),\, g(S) \big\}.
\end{equation}
This formulation captures the fundamental trade-off between quantitative performance (coverage) and qualitative consistency (alignment).

The formulation in \eqref{eq:ethiopia-multi-obj} defines an idealized bi-objective optimization problem that simultaneously maximizes population coverage $f(S)$ and alignment with expert advice $g(S)$. 
However, since the problem is intractable (there is no natural way to scalarize the two objectives) it is often necessary to provide guarantees with respect to the more  measurable coverage objective. 
To make the problem operational, we introduce a constrained bi-objective relaxation that preserves theoretical guarantees on coverage while allowing flexible adaptation to both coverage and advice alignment.

\paragraph{The $\alpha$--$\beta$ guarantee problem.}
To provide theoretical guarantees while maintaining flexibility in alignment, we introduce two control parameters $\alpha, \beta \in [0,1]$. 
Let $OPT_b$ denote the optimal allocation of size $b$ that maximizes the coverage function $f$. 
We then define the constrained optimization problem as:
\begin{align}
\nonumber
&\max_{S \subseteq V,\, |S| = b} \big\{ f(S),\, g(S) \big\},\\[2mm]
&\text{s.t.} \quad f(S) \ge (1 - e^{-\alpha \beta}) \, f(OPT_b).
\label{eq:ethiopia-alpha-beta}
\end{align}
The constraint ensures that the final allocation retains at least a fraction $(1 - e^{-\alpha \beta})$ of the optimal coverage value.  
By adjusting $\alpha$ and $\beta$, one can explicitly control the balance between maintaining strong theoretical coverage guarantees and achieving closer alignment with expert guidance.  
The parameter $\alpha$ defines the proportion of selections made by the greedy approach relative to the LLM-based approach, while $\beta$ controls the greedy approach’s flexibility to deviate from the LLM allocation.
The detailed operational meaning of these parameters is discussed further in Section~\ref{sec:method}.

\begin{remark}[Challenges of multi-objective optimization]
Classical scalarization methods---such as weighted sums or Pareto front exploration~\cite{miettinen1999nonlinear,hayes2022practical}---require predefined numerical weights for each objective, which are often unavailable when stakeholder preferences are expressed in natural language.  
Moreover, these approaches treat objectives as static, whereas expert opinions may evolve iteratively during the planning process.  
Our formulation differs by explicitly integrating language-based feedback within the optimization loop, thereby allowing the trade-off between coverage and alignment to be dynamically adjusted while preserving provable performance guarantees.
\end{remark}

\section{Proposed Method}
\label{sec:method}

\subsection{Algorithm overview}
Our framework integrates optimization algorithms with LLM-driven reasoning to balance two competing objectives: maximizing population coverage and aligning with qualitative expert advice as in Eq. \eqref{eq:ethiopia-multi-obj}.  
The overall process proceeds iteratively across five stages (Algorithm~\ref{alg:1}),
alternating between algorithmic updates that ensure theoretical guarantees and language-based refinements that incorporate human guidance.
At a high level, it contains 5 steps, i.e.,
Step 1: A classical greedy algorithm produces an initial coverage-maximizing (low level) grid-cell  allocation.  
Step 2: Given the grid-cell allocation, an LLM produces a (high-level) district allocation using expert advice expressed in natural language.  
Step 3: A constrained greedy procedure refines the district level allocation into a grid-cell allocation while maintaining an $\alpha$--$\beta$ coverage guarantee.  
Step 4: The LLM receives structured feedback to improve alignment.  
Step 5: The prompts themselves are iteratively optimized through a form of “prompt gradient descent.”  
We detail each step below.

\paragraph{\textbf{Step 1: Initial grid-cell Allocation (Greedy in line 3 of Algorithm~\ref{alg:1})}}
We first allocate facilities to grid-cells using the standard greedy algorithm of \citet{nemhauser1978analysis}, which provides a $(1 - 1/e)$-approximation for maximizing monotone submodular functions.  
The resulting allocation of facilities to grid-cells, $S_0$, serves as the initial baseline.  
From this grid-cell allocation, we extract the district allocation $h(S_0)$, which encodes how many facilities are assigned to each district (without the information about the exact cells), and will be used to guide the subsequent district-level updates.

\paragraph{\textbf{Step 2: LLM-Powered district  allocation (line 6 of Algorithm~\ref{alg:1})}}

Next, we refine the district allocation $d=h(S)$ using an LLM that processes both quantitative and qualitative contextual information.  
The model considers the current and previous district allocations, budget, population statistics, and a set of expert advice sentences $\bA$.  
To maintain stability, we restrict the LLM to modify the districts of at most two facilities per iteration.
The full prompt used in this step is in Section~\ref{app:prmpts} of the Appendix.

This step produces a revised district allocation that respects both expert intent and practical feasibility.  
A full description of the prompting strategy is provided in the Appendix \ref{app:prmpts}.

\paragraph{\textbf{Step 3: GuidedGreedy (line 7 of Algorithm~\ref{alg:1};Algorithm~\ref{alg:2})}}
We consider the district allocation $d$, as a vector of per-district budget and perform a guided greedy selection (Algorithm~\ref{alg:2}) to determine specific facility locations.  
In line 3 of Algorithm~\ref{alg:2}, we initialize the set $\hat{S}$ as an empty set.

In the loop on lines 4--12 of Algorithm~\ref{alg:2}, the algorithm greedily adds grid-cells to the set $\hat{S}$. 
On lines 5--6, we compute the cells with the maximum marginal gains, either ignoring or considering the district allocation $d$, respectively.  
The condition in line 8 ensures that at least $\lceil \alpha b\rceil$ cells attain a marginal gain of at least $\beta$ times the maximum marginal gains computed with no restrictions (i.e., ignoring $d$).
A running example demonstrating Algorithm~\ref{alg:2} is given in Table~\ref{tbl:1}.

After obtaining the set $S_i$, we compute the values $f(S_i)$ and per-district contributions $f(S_i \cap T_j)$ for all $1\le j \le r$.  
These values constitute the quantitative feedback signals, $\Delta f$, and $\Delta h$ which are then used to guide the next LLM iteration.

\begin{table}[h]

    \centering
    \small 
    \setlength{\tabcolsep}{3pt}
    \caption{Running example of Algorithm~\ref{alg:2}, inputs: $\alpha=0.25, \beta=0.5, d=\{1:3,2:0\}$ \label{tbl:1}; columns: \textit{Iter} = algorithm iteration, \textit{$|S|$}= size of grid-cell allocation at iteration start, \textit{$d$}= per district budgets at iteration start, $c(L5)$ and $c_d(L6)$ = values in lines 5 and 6, \textit{L8 Check} = condition outcome (line 8), \textit{Action} = selected action (line 9 or 12).
 }
    \begin{tabular}{|c|c|c|c|c|l|l|}
        \hline
        \textbf{Iter} & \textbf{$|S|$} & \textbf{$d$}& \textbf{$c$ (L5)} & \textbf{$c_d$ (L6)} & \textbf{L8 Check} & \textbf{Action} \\
        \hline
        1 & 0 & $\{1:3,2:0\}$& $10$ & $4$ & $0 > 0$? \textbf{No} & Pick $c$ ($10; T_2$) \\
          &   &   &   &     & $4 \ge 0.5\cdot 10$? \textbf{No} & \\
        \hline
        2 & 1 &$\{1:3,2:0\}$ & $8$ & $4$ & $1 > 0$? \textbf{Yes} & Pick $c_d$ ($4; T_1$) \\
        \hline
        3 & 2 &$\{1:2,2:0\}$ & $8$  & $3$ & $2 > 0$? \textbf{Yes} & Pick $c_d$ ($3; T_1$) \\
        \hline
    \end{tabular}
\end{table}

\paragraph{\textbf{Step 4: Verbal Reinforcement via scalar, and vectors' Comparison. (lines 8-9 of Algorithm~\ref{alg:1})}}
We quantify the change in both coverage and allocation between iterations by computing:
\begin{equation}
\label{eq:delta}
\Delta f = f(S_{i+1}) - f(S_i), \quad \Delta h = h(S_{i+1}) - h(S_i).
\end{equation}
These differences form the basis of a new prompt to the LLM, requesting verbal feedback to improve future alignment.  
Specifically, the $Feedback$ in the prompt instructs the LLM to (i) describe the observed differences and (ii) propose adjustments that enhance both alignment with expert advice and overall coverage.

\paragraph{\textbf{Step 5: Prompt optimization. (line 10 of Algorithm~\ref{alg:1})}}
Following \cite{xiong2025vortex}, we replace conventional parameter-space gradient descent with an analogous update in the prompt space.  
We decompose each prompt into a static and editable part:
\begin{equation}
\label{eq:prompt1}
\text{Prompt}_i = P_{\text{Fix}} \,\|\, P_{\text{Editable},i},
\end{equation}
where $P_{\text{Fix}}$ encodes the task template and $P_{\text{Editable},i}$ accumulates iteration-specific refinements.  
After receiving new feedback, we update
\begin{equation}
\label{eq:prompt2}
P_{\text{Editable},i+1} = P_{\text{Editable},i} + Feedback_{i+1},
\end{equation}
yielding the new prompt $\text{Prompt}_{i+1} = P_{\text{Fix}} \,\|\, P_{\text{Editable},i+1}$.  
This procedure gradually improves prompt quality and alignment consistency over time.
The high level of the suggested approach is given in Algorithm~\ref{alg:1}.



\begin{algorithm}
\caption{LLM-Enhanced District Resource Allocation}
\begin{algorithmic}[1]
\State \textbf{Input:} Budget $b$, advice set $\mathbf{A}$, balance parameters $\alpha,\beta \in [0,1]$, available grid-cells $V$, prompts $Prompt_1$, $Prompt_{reflection}$.
\State \textbf{Output:} Grid-cell allocation $S_{limit}$.
\State (Init) $S_0\gets Greedy(b,V)$
\State (Init) Update $Prompt_1$ with $S_0$, $h(S_0)$ and the other inputs.
\For{$i = 1$ to limit}
\State (LLM) $d\gets$ LLM(Prompt$_i$)
\State (Solve)  $S_i \gets GuidedGreedy(\alpha,\beta, b, V,\emptyset, d)$.
\State (Compare) Compute differences $\Delta f$, $\Delta h$ (Eq.~\ref{eq:delta}) and update $Prompt_{reflection}$.
\State (Verbal Feedback) Query LLM to generate verbal reflection.
\State (Optimize Prompt) Update Prompt$_{i+1}$, based on Eq. \eqref{eq:prompt1}, \eqref{eq:prompt2}.
\EndFor
\State \textbf{Return:} $S_{limit}$
\end{algorithmic}
\label{alg:1}
\end{algorithm}

\begin{algorithm}
\caption{\textsc{GuidedGreedy}$(\alpha,\beta,b,V,d)$}
\begin{algorithmic}[1]
\State \textbf{Input:} Parameters $\alpha,\beta \in [0,1]$; budget $b\in \mathbb{N}_{>0}$; available cells $V$;existing district allocation $S$, district budgets $d \in \mathbb{N}^r$.
\State \textbf{Output:} Grid-cell allocation $\hat{S}$.
\State $\hat{S}\gets S$
\While{$|\hat{S}| < b$}
    \State $c \gets \arg\max_{v \in V}\big(f(\hat{S} \cup \{v\}) - f(\hat{S}) \big)$
    \State $c_d \gets \arg\max_{v \in \bigcup_{j: d_j>0} T_j}\big( f(\hat{S}  \cup \{v\}) - f(\hat{S}) \big)$
    \State Let $j_d$ be the district such that $c_d \in T_{j_d}$
    \If{$|\hat{S}| > \lceil \alpha b \rceil$ \textbf{or} 
    \Statex $f(\hat{S} \cup \{c_d\}) - f(\hat{S}) \ge \beta \cdot (f(\hat{S} \cup \{c\}) - f(\hat{S}))$}
        \State $\hat{S} \gets \hat{S} \cup \{c_d\}$
        \State $d_{j_d} \gets d_{j_d} - 1$
    \Else
        \State $\hat{S} \gets \hat{S} \cup \{c\}$
    \EndIf
\EndWhile
\State \Return $\hat{S}$
\end{algorithmic}
\label{alg:2}
\end{algorithm}

\begin{remark}[A heuristic improvement for $\beta=1.0$]
When $\beta = 1.0$, we use the following heuristic to improve the performance of Algorithm~\ref{alg:1}: before line $6$, we run the Greedy algorithm to allocate $\lceil\alpha \cdot b\rceil$ facilities, and provide this allocation to the LLM as a part of the prompt in line $6$ (the greedy allocation is serving as a minimum district allocation, like we do in line 6 of Algorithm~\ref{alg:3}). We then skip line $7$ and continue with line $8$. This approach allows the LLM to leverage the greedy selection while preserving the same theoretical guarantee.

\end{remark}

\subsection{Theoretical guarantee}

Since the coverage function $f$ is monotone and submodular, the standard greedy algorithm achieves a $(1 - 1/e)$ approximation ratio \cite{nemhauser1978analysis}.  
We extend this result to our hybrid framework, which combines greedy selection with LLM-guided refinement.

Theorem~\ref{thm:1} generalizes the classic result of \cite{hochbaum1998analysis}.  
While that work establishes approximation guarantees for allocations closely following a purely greedy strategy, we extend the analysis to two dimensions:  
(i)~arbitrary monotone submodular functions (beyond coverage-specific objectives), and  
(ii)~partial selection settings where fewer than $k$ cells are chosen.  
This latter generalization is crucial in our setup, as our iterative framework allows the LLM to adjust or supplement a subset of the greedy-selected locations.

\begin{theorem}
Let $b \in \mathbb{N}_{>0}$ and $\alpha, \beta \in [0,1]$ be the balance parameters provided as inputs to Algorithm~\ref{alg:1}.
Let $OPT_b \subseteq V$ denote an optimal subset of size $b$ maximizing $f$.  
Denote by $S_{\text{limit}}$ the allocation returned by Algorithm~\ref{alg:1}.  
Then
\[
f(S_{\text{limit}}) \ge 
\big( 1 - (1 - \tfrac{\beta}{b})^{\alpha b} \big) f(OPT_b)
\ge
\big( 1 - e^{-\alpha \beta} \big) f(OPT_b).
\]
\label{thm:1}
\end{theorem} 

This result guarantees that even when part of the selection process is delegated to the LLM—subject to the $\alpha$–$\beta$ trade-off parameters—the final allocation still retains an exponential coverage bound comparable to the classical greedy algorithm.
The parameters $\alpha$ and $\beta$ are provided to Ethiopian stakeholders as tools to determine the optimal balance between offering guarantees and adhering to expert advice. 
We anticipate that different regions in Ethiopia will select varying guarantee parameters, depending on their specific need to balance the influence of current human planners with coverage considerations.
The parameter $\beta$ is particularly useful when the coverage values of different cells are similar. 
In such cases, it prevents the greedy algorithm from selecting a cell solely because it offers a marginally better coverage gain despite much worse alignment with the advice. 
The parameter $\alpha$ allows the LLMs some flexibility to follow the advice even when it recommends cells that provide very low coverage improvement.

\paragraph{An Example}
\begin{figure}
    \centering
    \includegraphics[width=0.5\textwidth]{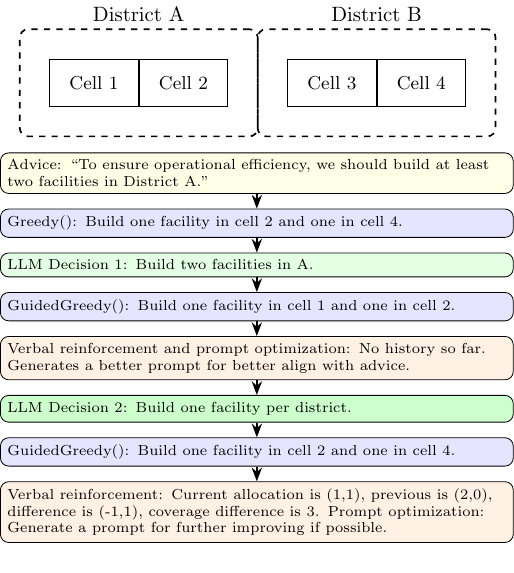} 
    \caption{Toy running example with two districts and four cells; 
    For GuidedGreedy, we use the parameters $\alpha=0,\beta=1$ allowing the LLM to fully determine the district allocation. Greedy() denotes the greedy algorithm of ~\cite{nemhauser1978analysis}.
    }
    \Description{Flowchart showing a two-district grid with a detailed running example. }
    \label{fig:example}
\end{figure}

In Figure~\ref{fig:example}, we present a toy running example of our approach. The example considers a region consisting of two districts, each containing two grid-cells. Each grid-cell represents a 1 km × 1 km area on a map of Ethiopia, and our approach focuses on selecting districts and cells for building health facilities.
Given an advice sentence, we first apply the algorithm of~\cite{nemhauser1978analysis} to obtain a grid-cell allocation. 
We then invoke the LLM to produce a district allocation, incorporating this grid-cell allocation among other factors.
Based on this district allocation, we run the GuidedGreedy algorithm to compute a refined grid-cell allocation. 
Verbal reinforcement and prompt optimization are used to improve the prompt, and the updated prompt is supplied to the LLM in the subsequent iteration. This procedure repeats until a final grid-cell allocation is obtained.
Note that typical regions in Ethiopia generally have many more districts.

\subsection{Supporting online budget}
\label{sec:online}
An important extension of our framework considers the setting where budgets arrive sequentially over multiple rounds \cite{choo2026optimizing}, for example, one allocation round per fiscal year.  
We assume that all candidate grid-cells remain available throughout the planning horizon and that new budget increments become available at the beginning of each round.
To adapt Algorithm~\ref{alg:1} to this setting, we modify its update rule to incorporate the historical allocations from previous rounds.  
Specifically, in line~6 of Algorithm~\ref{alg:3}, the prompt to the LLM now includes a \emph{minimum district allocation vector} representing facilities already selected in earlier years.  
This ensures that the LLM respects all previously established allocations and builds upon them when optimizing the new round.

Let $b_r$ denote the available budget in round $r$, and suppose the process has reached round $t$.  
The cumulative number of facilities planned up to this point is therefore $b = \sum_{r=1}^{t} b_r$.
The LLM prompt for round $t$ requests an updated allocation consistent with this cumulative budget and the existing district allocation.
The multi-round variant of our approach is summarized in Algorithm~\ref{alg:3}.

\begin{algorithm}
\caption{Multiple Rounds Facility Location}
\label{alg:3}
\begin{algorithmic}[1]
\State \textbf{Input:} Budgets $b_1,...,b_t$, advice set $\mathbf{A}$, balance parameters $\alpha,\beta \in [0,1]$, available grid cells $V$, existing facilities $S_e$, prompts $Prompt_0$, $Prompt_{reflection}$, time horizon $t$.
\State \textbf{Output:} Allocation $S$
\State $S\gets \emptyset$
\For{$r = 1$ to $t$}
\State (Init) $S_0\gets Greedy(b,V)$
\State (Init) Update $Prompt_0$ with $S_0$, $d(S_0)$, minimum district allocation $h(S)$, and the other inputs.
\For{$i = 1$ to limit}
\State (LLM) $d\gets$ LLM(Prompt$_i$).
\State (Solve)  $S_i \gets GuidedGreedy(\alpha,\beta, b_r, V,S, d)$.
\State (Compare) Compute differences $\Delta f$, $\Delta h$ (Eq.~\ref{eq:delta}) and update $Prompt_{reflection}$.
\State (Verbal Feedback) Query LLM \\\;\;\;\;\;\;\;\;\;\;\;\;\;\;\;\;to generate verbal reflection.
\State (Optimize Prompt) Update Prompt$_i$, \\\;\;\;\;\;\;\;\;\;\;\;\;\;\;\;\;based on Eq. \eqref{eq:prompt1}, \eqref{eq:prompt2}.
\EndFor
\State $S\gets S_{limit}$
\EndFor
\State \textbf{Return:} $S$
\end{algorithmic}
\end{algorithm}

The following theorem expands the results of Theorem~\ref{thm:1} to the online settings. 

\begin{theorem} Let $b_1, \dots, b_t$ denote the budgets, and let $\alpha, \beta \in [0,1]$ be the balance parameters provided as inputs to Algorithm~\ref{alg:3}. Let $b=\sum_{i=1}^t b_i$ and let $OPT_b \subseteq V$ denote an optimal subset of size $b$ maximizing $f$.  Denote by $S$ the allocation returned by Algorithm~\ref{alg:3}.  Then \[ f(S) \ge \big( 1 - (1 - \tfrac{\beta}{b})^{\alpha b} \big) f(OPT_b) \ge \big( 1 - e^{-\alpha \beta} \big) f(OPT_b). \] \label{thm:2} \end{theorem}     

\begin{remark}
This online-budget formulation allows the model to handle dynamically arriving resources while maintaining continuity across years.  
In practice, this property is essential for national-scale public health planning, where upgrades are deployed incrementally based on evolving financial and logistical constraints.  
The framework thus generalizes naturally from a single-shot allocation problem to an online, multi-horizon setting, preserving the same theoretical guarantees on coverage at each stage.
\end{remark}

\section{Experiments}

\subsection{Dataset}

We evaluate the proposed framework using Ethiopia’s projected 2026 population data \cite{worldpop_eth_2015_2030}. Walking accessibility is computed following \cite{choo2026optimizing} with the global friction map \cite{weiss2020global}, assuming a maximum two-hour walking distance. The experiments cover three regions—Afar, Somali, and Benishangul Gumuz—characterized by diverse terrain and population densities.
To emulate multi-stakeholder guidance, we queried Gemini-2.5-Pro to generate 20 expert advice sentences grouped into four sets (five per expert) with intentional contradictions to mimic real decision environments. For quantitative evaluation, an alignment function was automatically synthesized by Gemini-2.5-Pro to score allocations between 0 and 1 according to the textual advice. All iterative allocation and feedback steps were executed using Gemini-2.5-Flash as shown in Figure~\ref{fig:scheme}.
The used prompts are available in Section~\ref{app:prmpts} of the Appendix.

\subsection{Experiment 1: Is the verbal feedback useful?}

In Figures~\ref{fig:vq_afar},\ref{fig:vq_gumuz},\ref{fig:vq_somali} we present the results for two variants of our approach: i) \textbf{Quantitative feedback}, where alignment signals are numeric; and ii) \textbf{Verbal feedback}, where updates are guided by natural-language reflections from the LLM. Across all three regions, the verbal-feedback variant achieved notably higher advice alignment, confirming that language-based reasoning effectively internalizes qualitative goals. The quantitative variant, in contrast, yielded slightly higher coverage values, reflecting its direct optimization focus. Importantly, both variants consistently satisfied the theoretical 
$1-e^{-\alpha\beta}$ coverage guarantee (even if we assume that the optimal selection, $OPT_b$, covers the entire population). Among the three regions, Benishangul Gumuz exhibited the largest gap between the two modes, suggesting stronger contextual heterogeneity where language input provides added value.
\begin{remark}
Note that in all regions, the coverage decreases as the alignment with advice increases. We justify this decrease by the additional knowledge that the human expert provides. 
For example, we may prioritize improving health outcomes in districts with poor existing health status, even if it results in a lower overall coverage.
\end{remark}

\begin{figure}
        \centering
\includegraphics[width=0.5\textwidth]{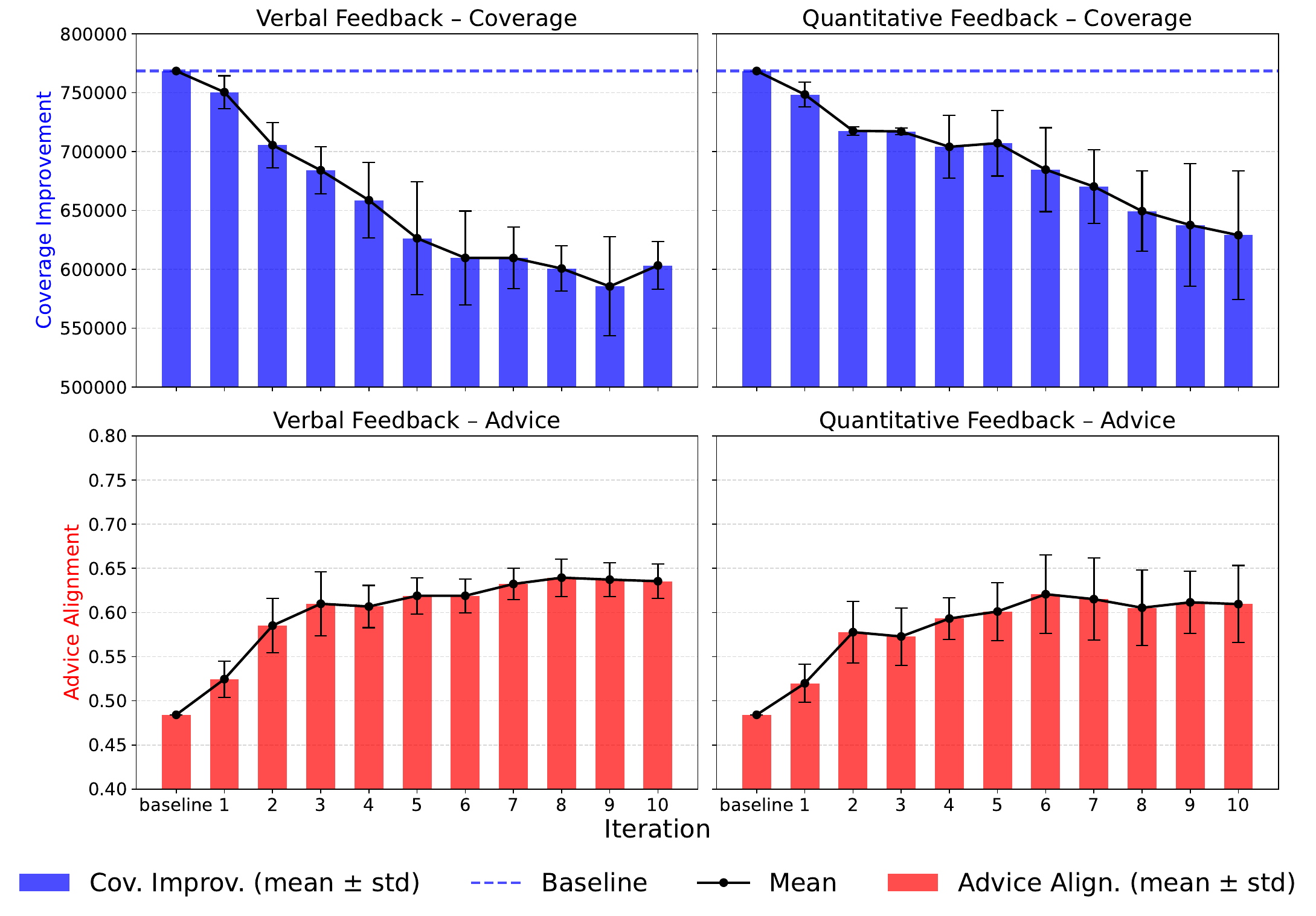}
        \caption{Verbal vs. quantitative feedback for Afar; $\alpha=0.0$, $\beta=1.0$. }
        \label{fig:vq_afar}     
        \Description{ Bar chart comparing verbal and quantitative feedback for Afar. }
\end{figure}

\begin{figure}
    \centering
    \includegraphics[width=0.5\textwidth]{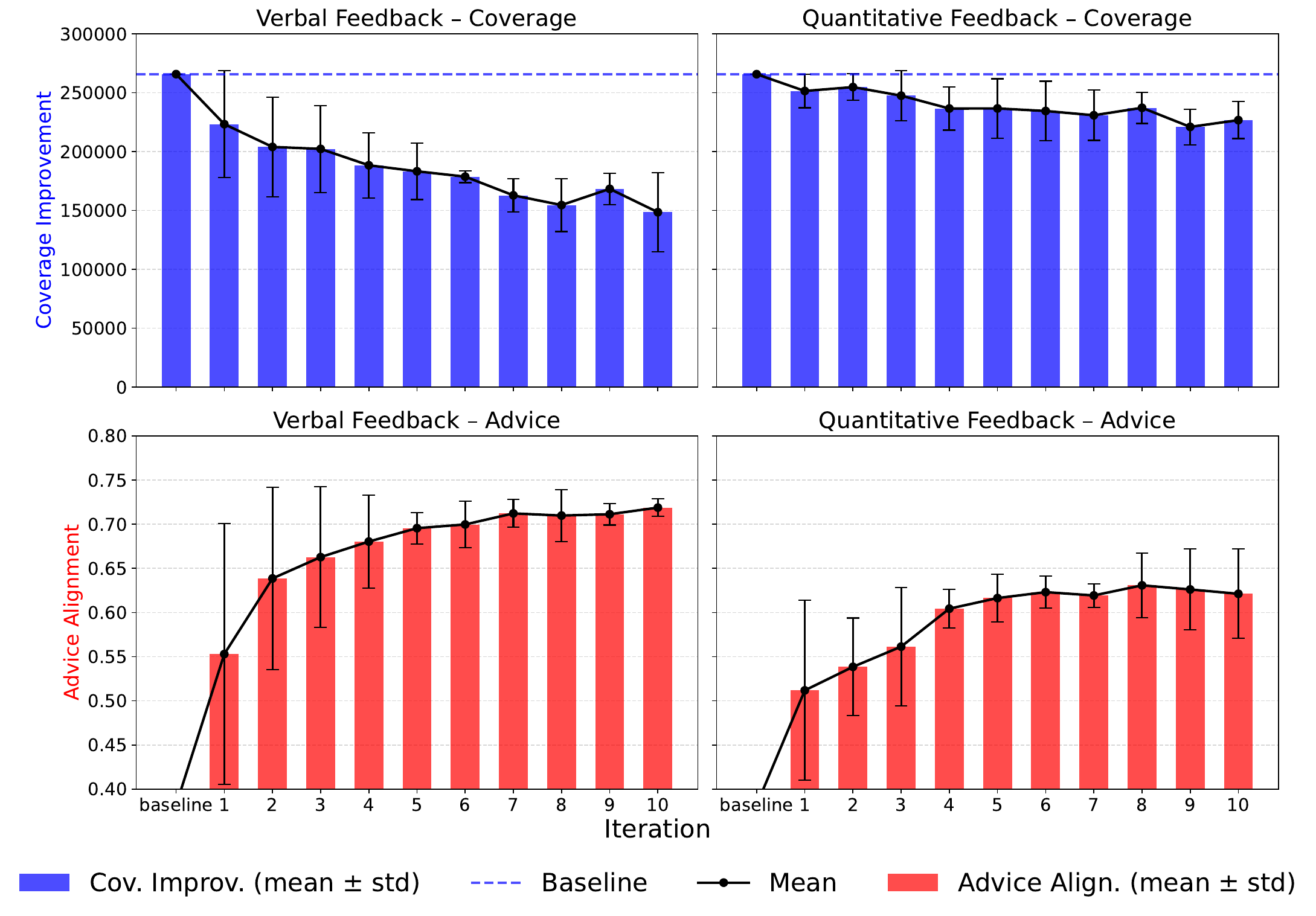}
    \caption{Verbal vs. quantitative feedback for Benishangul Gumuz; $\alpha=0.0$, $\beta=1.0$.}
    \Description{ Bar chart comparing verbal and quantitative feedback for Benishangul Gumuz. }
    \label{fig:vq_gumuz}
\end{figure}

\begin{figure}
    \centering    \includegraphics[width=0.5\textwidth]{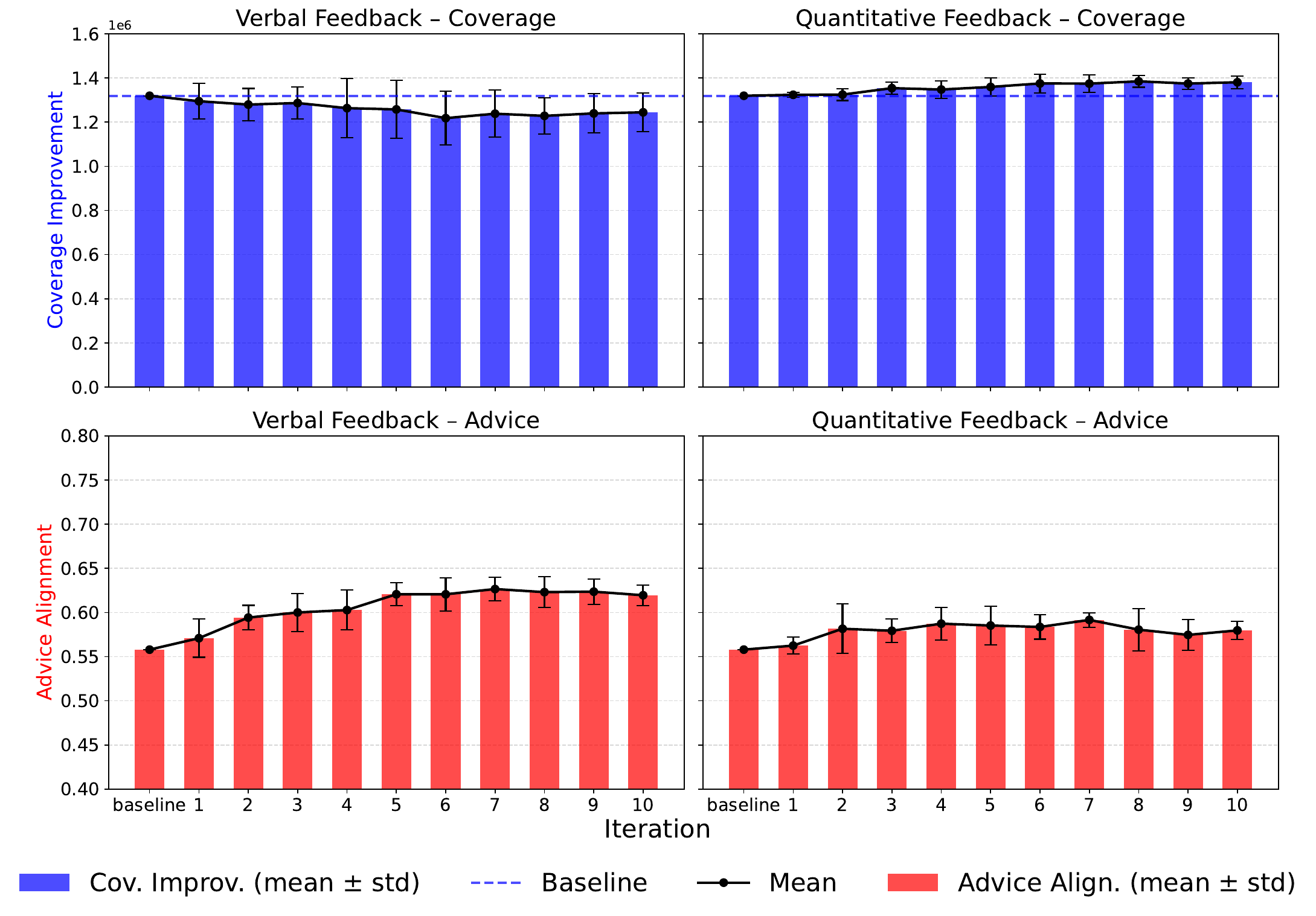}
    \caption{Verbal vs. quantitative feedback for Somali; $\alpha=0.0$, $\beta=1.0$.}
    \label{fig:vq_somali}
    \Description{ Bar chart comparing verbal and quantitative feedback for Somali. }
\end{figure}

\subsection{Experiment 2: Change in actual coverage and alignment as the parameter \texorpdfstring{$\alpha$}~ increases}
\begin{figure}[htbp]
    \centering
\includegraphics[width=0.45\textwidth]{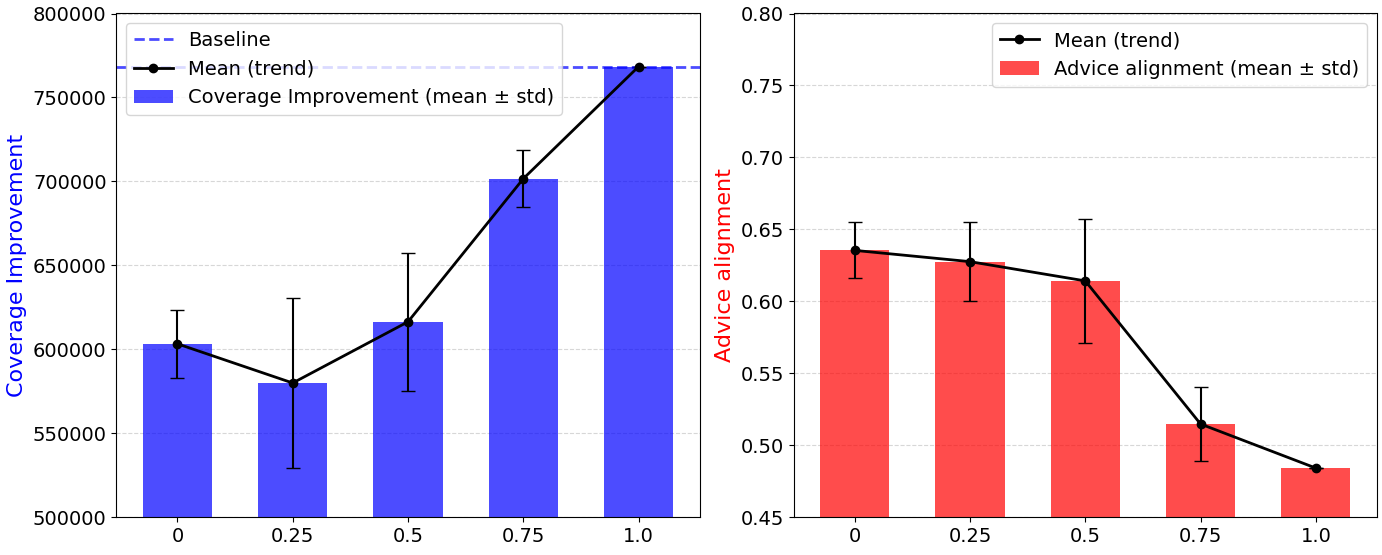} 
\includegraphics[width=0.45\textwidth]{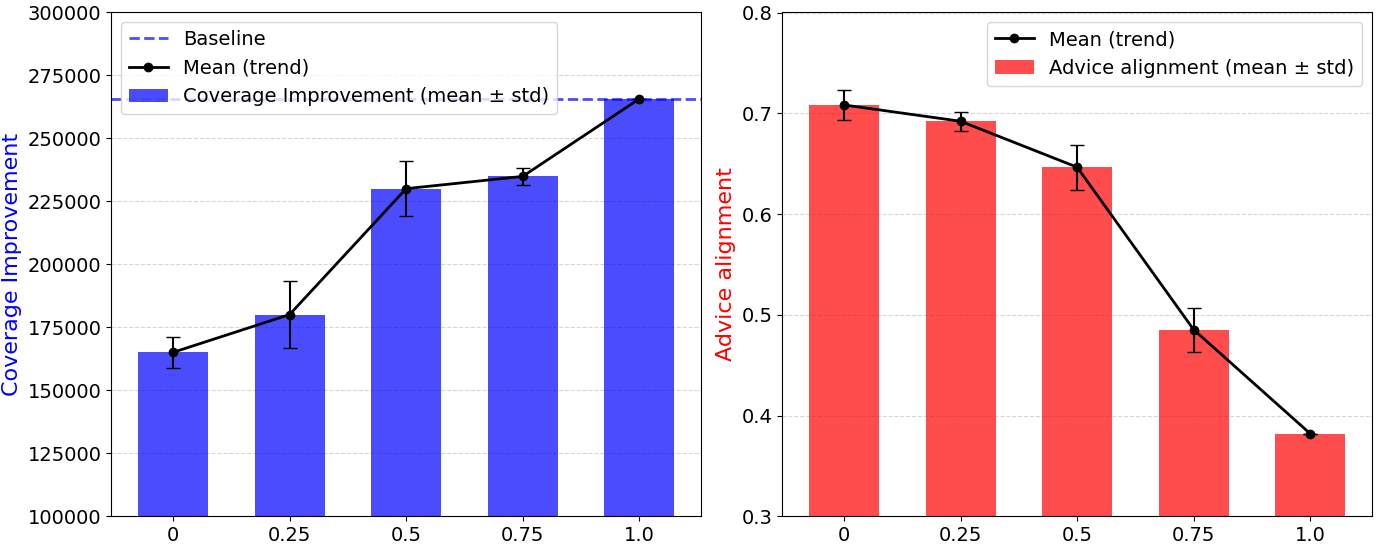}
\includegraphics[width=0.45\textwidth]{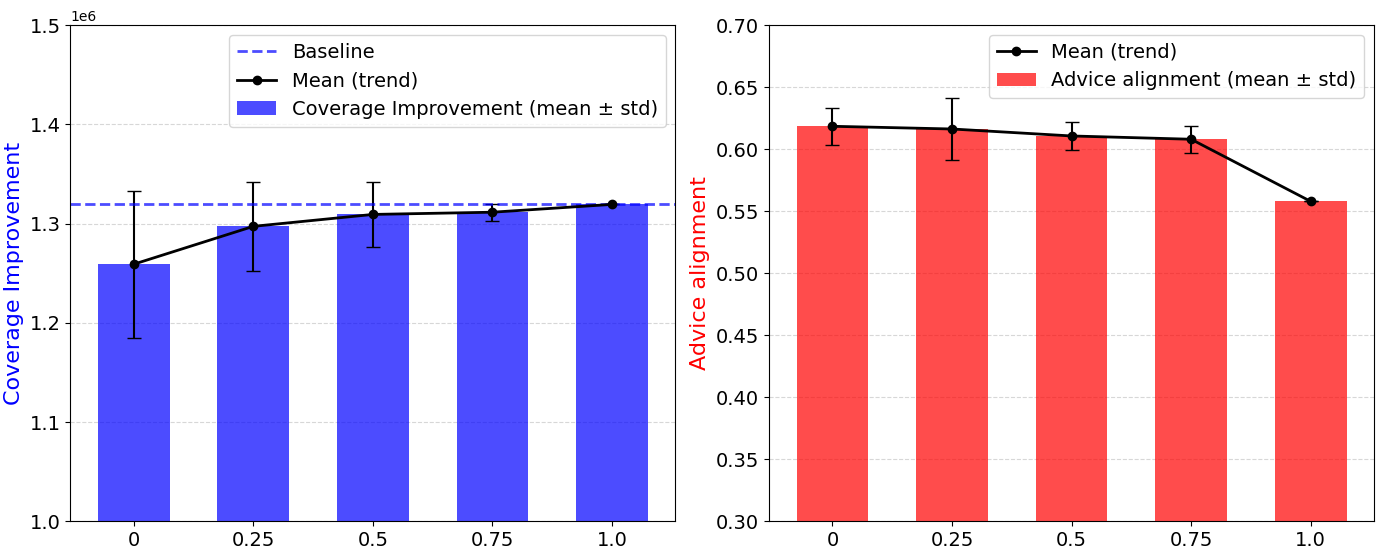}
\caption{Coverage and alignment for different values of $\alpha$ in Afar (top), Benishangul Gumuz (middle), and Somali (bottom) regions; $\beta=1.0$.}
\Description{ Bar charts for different values of $\alpha$ }
    \label{fig:alpha}
\end{figure}

In Figure~\ref{fig:alpha}, we measured changes in coverage and alignment to advice as a function of the parameter $\alpha$, while keeping $\beta=1$. 
As expected, larger values of $\alpha$ lead to higher coverage improvement but lower advice alignment. 
One exception is that in the Afar region, the coverage in $\alpha=0.0$ is worse than in $\alpha=0.25$, indicating that, in practice, setting $\alpha=0.25$ offers little benefit for coverage, since the LLM iterative approach achieves similar results without enforcing greedy selections. 
One thing to note is that as the greedy algorithm is deterministic, there are no variance bars for $\alpha=1$. These results validate the interpretability of 
$\alpha$: small 
$\alpha$ values promote qualitative conformity, while larger 
$\alpha$ values recover the classical submodular-optimization behavior.

\subsection{Experiment 3: Is it useful to have longer time window?}
\label{sec:wind}
In Figure~\ref{fig:window} in the Appendix, we examine whether incorporating a longer feedback history improves the performance of verbal refinement. Specifically, we compare two variants: one that includes information from the past three iterations (\emph{3-step window}) and another that considers only the immediately preceding iteration (\emph{1-step window}).
In the \emph{no-window} variant, the feedback prompt contained the current and previous allocations but excluded their differences in allocation and total coverage. Overall, after ten iterations, the observed differences between the two best variants (1-step and 3-step windows) were minimal in both population coverage and advice alignment. This suggests that the LLM effectively internalizes historical trends even with limited contextual information, making explicit multi-step computation largely redundant.

Furthermore, extending the feedback window to three iterations did not yield measurable gains in either metric. Consequently, we adopt the 1-step window setting as our default configuration, as it achieves comparable performance while being slightly more stable and computationally efficient.

\subsection{Analyzing visually the different allocation algorithms}
Figure~\ref{fig:visual} presents the resulting facility allocations in a subregion of the Afar region, generated under two parameter settings: $\alpha = 0, \beta = 1$ and $\alpha = 0.5, \beta = 1$.
This visualization highlights how different trade-off configurations lead to distinct spatial allocation patterns.
While the two solutions exhibit substantial overlap in high-priority areas—indicating agreement on core population centers—the overlap is not complete.
The $\alpha = 0.5$ configuration, which emphasizes coverage optimization, tends to expand toward denser clusters, whereas the $\alpha = 0$ configuration better reflects qualitative expert preferences.
Together, these contrasting patterns illustrate the interpretability and tunability of the proposed framework in balancing quantitative coverage with qualitative alignment.

\begin{figure}
    \centering
    \includegraphics[width=0.3\linewidth,angle=90]{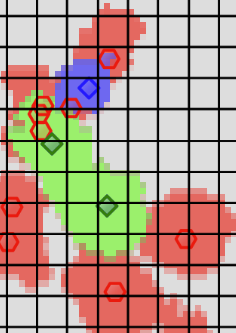}
        \caption{Red hexagons indicate existing facilities. Green rhombuses show locations selected by the algorithm with $\alpha = 0.5$, while blue rhombuses correspond to $\alpha = 0$ ($\beta$ is always $1.0$). Colored areas represent the accessibility areas of the facilities. Each grid cell covers 5 km × 5 km.}
    \label{fig:visual}    
    \Description{ Gridded map showing facility allocations in a subregion of the Afar region.}
\end{figure}

\subsection{Summary of findings}
Overall, the experimental results demonstrate that incorporating language-based feedback substantially enhances interpretability without compromising the theoretical coverage guarantees of the underlying optimization framework. The trade-off parameter $\alpha$ provides a clear and intuitive mechanism for balancing quantitative coverage performance with qualitative expert alignment, enabling flexible policy tuning across regions. Moreover, the analysis shows that short feedback windows are sufficient for stable convergence, indicating that the LLM can effectively capture temporal patterns without extensive historical context. Finally, the visual comparisons reveal that different parameter settings produce distinct yet interpretable allocation patterns, offering planners and policymakers valuable insights into the trade-offs between efficiency and fairness in real-world deployment.
\section{Conclusions}

In this paper, we present a novel framework for optimizing bi-objective problem in a complex environment. 
Our LEG framework employs an adaptive greedy algorithm integrated within an LLM-based refinement loop, ensuring coverage guarantees while jointly optimizing advice alignment and coverage.
We evaluated the framework on three real-world datasets, demonstrating its effectiveness. 
Combined with existing literature, our work brings us closer to addressing pressing real-world challenges.


\clearpage
\section*{Acknowledgements}

We are very grateful to Kasahun Sime and Wondesen Nigatu of the  Ministry of Health, Ethiopia, for their continuous support and valuable advice.
The findings and conclusions in this report are those of the authors and do not represent any recommendations
or official positions.

\bibliographystyle{ACM-Reference-Format} 
\bibliography{refs}
\appendix
\onecolumn

\section{Prompts}
\label{app:prmpts}
Here we provide the prompts used in our experiments. For generating the advice we used the prompt:

\begin{tcolorbox}[colback=gray!5,colframe=black!40]
\textbf{Task:} Generate 20 sentences of (possibly contradicting) advice on where to place the facilities for the given budget.\\
\textbf{Input:} Types={}, a budget=20 health facilities.\\
\textbf{Specifications:} The advice is given by multiple experts/advisors and there should be contradictions and different points of view. Each sequence of 5 sentences belong to a different expert.\\
\textbf{Goal:} The advice will be used for deciding how many facilities to build in each district.\\
\textbf{Output:} A python pair that contains:\\
1. The 20 advice sentences, numbered from 1 to 20, with no additional information.\\
2. Two contradicting sentences and an explanation why they contradict.
\end{tcolorbox}

The advice generation is done by the strong model: Gemini-2.5-pro.
For evaluating the alignment with advice we use the following prompt with the same strong model:

\begin{tcolorbox}[colback=gray!5,colframe=black!40]
\textbf{Task:} Generate an alignment function.\\
\textbf{Inputs:} Advice=, sample allocation vector=. \\     
\textbf{Output:} A function: texttt{advice\_verifier(allocation\_vector)} 
that gets an allocation vector of the same format as an input and return an allocation value between 0 and 1.\\
\textbf{Goal:} The function should assess alignment with each advice sentence individually and then aggregate these assessments into a single overall score.\\
The function should consider partial alignment and not only full alignment with the advice sentences.\\
\textbf{Output:} a runnable python code of the function."
\textbf{Constraints:} The output should include only the function itself with no additional text. The function should be parsable by python exec function.
\end{tcolorbox}

The original allocation is selected as the results of an unconstrained greedy algorithm.
For adjusting the advice we use the following prompt with a weaker model (Gemini-2.5-flash):

\begin{tcolorbox}[colback=gray!5,colframe=black!40]
\textbf{\underline{Prompt for generating the high-level allocation:}}\\
\textbf{Task:} Generate a district allocation.\\
\textbf{Inputs:} types={}, district names={}, budget={}, advice={}, population per district={}, initial coverage per district={}, current allocation={}, previous allocation={}, minimum allocation={}.\\
\textbf{Goal:} align the allocation with advice sentences while improving population coverage.\\
\textbf{Constraints:} 1. Minimum allocation: no district may receive fewer units than specified in the minimum district allocation for this district (this constraint is critical). \\
2. Format: the new allocation must match the format of the current district allocation.\\
3. Budget: total units must equal {}. Move either one or two units between districts to improve alignment and coverage, without violating constraint 1.\\
4. Technical: all values must be integers, all districts included as keys (even if 0), and output must be valid Python dictionary syntax.\\
\textbf{Instruction:} Output only the updated district allocation as a Python dictionary. No explanations, comments, or extra text.\\
\textbf{Additional instructions:} FEEDBACK.
\end{tcolorbox}

We then construct a feedback addition to the prompt.  
The feedback prompt is generated by:

\begin{tcolorbox}[colback=gray!5,colframe=black!40]
\textbf{Task:} Describe the differences between the current and previous run and give instructions to improve the current allocation by better aligning with the  advice sentences and by improving coverage.\\
\textbf{Instructions:} 
1. Weigh the alignment of the current allocation and the previous allocation with each of the advice sentences and provide instructions to modify the current allocation so that the sum of contributions in the new suggested allocation is better than in both the current and previous allocations.\\
2. It is allowed to move either one or two units between districts to improve alignment and coverage, without violating constraint 1. \\
3. If impossible to find better allocation, suggest to retain the current allocation.\\                                        
\textbf{Input:} Advice={}, types={}, current allocation={}, previous allocation={}, allocation difference={}, coverage difference={}.\\ 
\textbf{Output format:} Start directly with the description and instructions. Be concise. 
\end{tcolorbox}

If there is no previous allocation, the text below is used as the feedback:
\begin{tcolorbox}[colback=gray!5,colframe=black!40]
Improve the previous solution with respect to following the advice and with respect to total coverage.
\end{tcolorbox}

\subsection{An example for prompt optimization}
An example for the generated feedback (Shortened and anonimized) is:
\begin{tcolorbox}[colback=gray!5,colframe=black!40]
The current allocation serves 14 unique districts, an improvement over the previous allocation's 12 districts, indicating better coverage.
**Comparison of Current vs. Previous Allocation with Advice Sentences:**...

**Comparison of Current vs. Previous Allocation with Advice Sentences:**\\
**Summary:** The current allocation is generally better than the previous, especially in terms of coverage (14 vs 12 districts) and alignment with several specific advice points (e.g., \_ minimum, \_ presence, \_ total and specific allocation, \_ coverage). However, it still falls short on key advice sentences.

**Instructions for Improvement:**

To improve the current allocation by better aligning with advice sentences and improving coverage, while moving one or two units without violating existing constraints (especially \_ /Zone \_ >= 7), here is a suggested modification:

1.  **Move 1 unit from \_ to \_.**
    *   This addresses **Advice 6 (Every zone >=3, esp. \_)** by bringing \_'s total facilities from 2 to 3.

2.  **Move 1 unit from \_ to \_.**

**New Suggested Allocation:** (district names have been omitted).\\
*   `\_`: 1 -> **0** \\
*   `\_`: 0 -> **1** \\ 
*   `\_`: 3 -> **2** \\
*   `\_`: 0 -> **1** 

(All other districts retain their current allocation)\\
...\\
This suggested allocation demonstrates a better sum of contributions by directly addressing critical advice gaps (\_'s minimum, \_'s internal focus) and improving coverage, while adhering to the constraint of moving only two units.
sum of allocated units 20.
\end{tcolorbox}

This text is appended to the prompt mentioned above as the Additional instructions section.

\subsection{Qualitative feedback:}
We tested a scenario where the feedback is not processed as a separate query. In this scenario, the feedback is as below:

\begin{tcolorbox}[colback=gray!5,colframe=black!40]
 1. Weigh the alignment of the current allocation and the previous allocation with each of the advice sentences and modify the current allocation so that the sum of contributions in the new suggested allocation is better than in both the current and previous allocations.\\
2. It is allowed to move either one or two units between districts to improve alignment and coverage, without violating constraint 1. \\
3. If impossible to find better allocation, suggest to retain the current allocation.\\
Allocation difference={}, coverage difference={}. "
\end{tcolorbox}

\section{Proofs of Theorems \ref{thm:1} and \ref{thm:2}}
Note that Theorem \ref{thm:1} is a special case of Theorem \ref{thm:2} where the time horizon 
t equals 1. 
Therefore, it suffices to prove Theorem \ref{thm:2}, from which the proof of Theorem \ref{thm:1} follows directly.
While similar proofs can be found in the literature (e.g., in \cite{nemhauser1978analysis} and \cite{hochbaum1998analysis}), we present the complete proof here to ensure our results are self-contained.

We begin with two necessary definitions:

\begin{definition}
Given a sequence of grid-cells $S$ and a subsequence $(s_{\beta_1}, \ldots, s_{\beta_K})$, we denote by $S_{\leq \beta_i}$ the prefix of $S$ up to and including $s_{\beta_i}$.
\end{definition}

\begin{definition}
Given a sequence of grid-cells  $S$ and a subsequence $(s_{\beta_1}, \ldots, s_{\beta_K})$, we denote by $S_{< \beta_i}$ the prefix of $S$ whose last element is the element immediately preceding $s_{\beta_i}$ in $S$.
\end{definition}

To prove Theorem~\ref{thm:2}, we proceed in three steps. First, we establish a key property of the grid-cell allocation returned by Algorithm~\ref{alg:2}. Second, we prove that any allocation with this property satisfies a certain approximation bound relative to the optimal solution. 
Finally, we show that Algorithm~\ref{alg:3} produces a grid-cell allocation meeting the requirements of Theorem~\ref{thm:2}.

\begin{lemma}
\label{lem:alg2connect}
Let $\hat{S}$ denote the grid-cell allocation returned by Algorithm~\ref{alg:2} 
when initialized with the input grid-cell allocation $S$. We abuse notation and 
refer to $\hat{S}$ as a sequence, following the insertion order.\footnote{Each invocation of Algorithm~\ref{alg:2} adds 
grid-cells in a specific order. When invoked multiple times, we order $\hat{S}$ 
by concatenating: first all grid-cells from the first invocation (in their 
insertion order), then all from the second invocation, and so on.} 

Define $K := \alpha b$, where $\alpha$ and $b$ are algorithm parameters. There exists a subsequence $S_{\beta} = (s_{\beta_1}, s_{\beta_2}, \ldots, s_{\beta_K})$ of $\hat{S} \setminus S$ of length $K$ such that for every $i \in [K]$, the marginal gain satisfies:
\begin{equation}
\label{eqn:gooditems}
    f(s_{\beta_i} \mid S_{<\beta_i}) \ge \beta \cdot \max_{y \in V \setminus S_{<\beta_i}} f(y \mid S_{<\beta_i}),
\end{equation}
where $S_{<\beta_i}$ denotes the sequence $\hat{S}$ immediately before $s_{\beta_i}$ is added.
\end{lemma}

\begin{proof}

In the first $K = \lceil \alpha b \rceil$ iterations of the loop in lines~4--12, the condition $|\hat{S}| > \lceil \alpha b \rceil$ (line 8) evaluates to false. Therefore, in each such iteration, exactly one of the following occurs:
\begin{enumerate}
    \item The condition $f(\hat{S} \cup \{c_d\}) - f(\hat{S}) \ge \beta \cdot \bigl(f(\hat{S} \cup \{c\}) - f(\hat{S})\bigr)$ holds, and $c_d$ is added to $\hat{S}$, or
    \item The condition does not hold, and $c$ is added to $\hat{S}$.
\end{enumerate}

\noindent \textbf{Case (i):} By the definition of conditional marginal gain,
\[
f(c_d \mid \hat{S}) = f(\hat{S} \cup \{c_d\}) - f(\hat{S}) \ge \beta \cdot \bigl(f(\hat{S} \cup \{c\}) - f(\hat{S})\bigr) = \beta \cdot f(c \mid \hat{S}).
\]
Since $c$ is selected in line 4 to maximize marginal gain, we have
\[
f(c \mid \hat{S}) = \max_{y \in V \setminus \hat{S}} f(y \mid \hat{S}).
\]
Therefore, setting $s_{\beta_i} = c_d$, we obtain
\[
f(s_{\beta_i} \mid S_{<\beta_i}) \ge \beta \cdot \max_{y \in V \setminus S_{<\beta_i}} f(y \mid S_{<\beta_i}).
\]

\noindent \textbf{Case (ii):} In this case, we have $s_{\beta_i} = c$. By the selection of $c$ in line 5,
\[
f(s_{\beta_i} \mid S_{<\beta_i}) = f(c \mid S_{<\beta_i}) = \max_{y \in V \setminus S_{<\beta_i}} f(y \mid S_{<\beta_i}) \ge \beta \cdot \max_{y \in V \setminus S_{<\beta_i}} f(y \mid S_{<\beta_i}),
\]
where the final inequality holds since $\beta \le 1$.

Thus, in both cases, the desired inequality holds for each of the first $K$ items added to $\hat{S}$, which form the subsequence $S_{\beta}$.
\end{proof}

The following lemma provides an approximation bound for sequences with a key property.

\begin{lemma}
\label{lem:bound}
Let $S$ be an ordered grid-cell allocation of size $b$, and let $K\in\mathbb{N}_{>0}$.
If there exists a subsequence $S_\beta = (s_{\beta_1}, \ldots, s_{\beta_K})$ of $S$ of length $K$ such that for each $i \in [K]$,
\[
f(s_{\beta_i} \mid S_{<\beta_i}) \ge \beta \cdot \max_{y \in V \setminus S_{<\beta_i}} f(y \mid S_{<\beta_i}),
\]
then
\[
f(S) \ge \left(1 - \left(1 - \frac{\beta}{b}\right)^{K} \right) f(OPT_b).
\]
\end{lemma}

\begin{proof}
Let $S_{\beta} = (s_{\beta_1}, \ldots, s_{\beta_K})$ be the subsequence from the lemma assumption.
Let $S_{\leq \beta_i}$ denote the prefix of $S$ with $s_{\beta_i}$ as its last grid-cell.
We denote $S_{\leq \beta_0} = \emptyset$ as the initial empty set.
Define the gap to optimality as:
\[
\Delta_i := f(OPT_b) - f(S_{\leq \beta_i}),
\quad i = 0,\ldots,K.
\]
Clearly, $\Delta_0 = f(OPT_b) - f(\emptyset) = f(OPT_b)$.

Fix any $i \in \{0,\ldots,K-1\}$.
Since some elements may have been added after $s_{\beta_i}$ and before $s_{\beta_{i+1}}$, $S_{\leq \beta_i}$ is a prefix of $S_{< \beta_{i+1}}$. 
By monotonicity, $f(S_{< \beta_{i+1}}) \ge f(S_{\leq \beta_i})$, and therefore the gap can only decrease:
\begin{equation}
\label{eq:mono}
f(OPT_b) - f(S_{< \beta_{i+1}}) \le \Delta_i.    
\end{equation}

Now we apply the standard greedy bound. By submodularity and monotonicity,
\[
f(OPT_b) - f(S_{< \beta_{i+1}}) \le b \cdot \max_{y \in V\setminus S_{< \beta_{i+1}}} f(y \mid S_{< \beta_{i+1}}).
\]
Since $s_{\beta_{i+1}} \in S_\beta$, by the lemma assumption,
\begin{equation}
\label{eq:9}
f(s_{\beta_{i+1}} \mid S_{< \beta_{i+1}}) \ge \beta \cdot \max_{y \in V\setminus S_{< \beta_{i+1}}} f(y \mid S_{< \beta_{i+1}}) \ge \frac{\beta}{b} \bigl( f(OPT_b) - f(S_{< \beta_{i+1}}) \bigr).
\end{equation}

We can now bound the gap after selecting $s_{\beta_{i+1}}$:
\begin{align*}
\Delta_{i+1} &= f(OPT_b) - f(S_{\leq \beta_{i+1}}) \\
&= f(OPT_b) - \bigl( f(S_{< \beta_{i+1}}) + f(s_{\beta_{i+1}} \mid S_{< \beta_{i+1}}) \bigr) && \text{(by definition)}\\
&= \bigl( f(OPT_b) - f(S_{< \beta_{i+1}}) \bigr) - f(s_{\beta_{i+1}} \mid S_{< \beta_{i+1}}) && \text{(rearranging)}\\
&\le \Bigl( 1 - \frac{\beta}{b} \Bigr) \bigl( f(OPT_b) - f(S_{< \beta_{i+1}}) \bigr) && \text{(by \eqref{eq:9})}\\
&\le \Bigl( 1 - \frac{\beta}{b} \Bigr) \Delta_i. && \text{(by \eqref{eq:mono})}
\end{align*}

Iterating this inequality for $i=0,\ldots,K-1$ yields:
\[
\Delta_K := f(OPT_b) - f(S_{\leq \beta_K})
\le \Bigl(1 - \frac{\beta}{b}\Bigr)^K f(OPT_b).
\]

Rearranging,
\[
f(S_{\leq \beta_K})
\ge \Bigl(1 - \Bigl(1 - \frac{\beta}{b}\Bigr)^K \Bigr) f(OPT_b).
\]

Since $S_{\leq \beta_K} \subseteq S$ (as $S_\beta$ is a subsequence of $S$), by monotonicity, $f(S) \ge f(S_{\leq \beta_K})$. Therefore,
\[
f(S)
\ge
\Bigl(1 - \Bigl(1 - \frac{\beta}{b}\Bigr)^K \Bigr) f(OPT_b).
\]
\end{proof}

We are now ready to prove the Theorem~\ref{thm:2}.

\begin{proof}
Let $b_1, \ldots, b_t$ denote the budgets for Algorithm~\ref{alg:3}.
For each call to Algorithm~\ref{alg:2} in line~9 of Algorithm~\ref{alg:3} with $r = \ell$ and $i = \text{limit}$, Lemma~\ref{lem:alg2connect} guarantees the existence of a subsequence 
\[
S^\ell_\beta = (s^\ell_{\beta_1}, s^\ell_{\beta_2}, \ldots, s^\ell_{\beta_{\alpha \cdot b_\ell}})
\]
of $\hat{S} \setminus S$ such that for every $j \in [\lceil\alpha \cdot b_\ell\rceil]$, the marginal gain satisfies:
\[
f(s^\ell_{\beta_j} \mid s^\ell_{<\beta_j}) \ge \beta \cdot \max_{y \in V \setminus s^\ell_{<\beta_j}} f(y \mid s^\ell_{<\beta_j}),
\]
where $s^\ell_{<\beta_j}$ denotes the prefix $(s^\ell_{\beta_1}, \ldots, s^\ell_{\beta_{j-1}})$.

We construct the sequence $S_\beta = S^1_\beta \cdot S^2_\beta \cdot \ldots \cdot S^t_\beta$ by concatenating these subsequences. 
By the guarantee above, each grid-cell in $S_\beta$ satisfies the marginal gain condition relative to the grid-cells preceding it in the concatenated sequence.
The total size of $S_\beta$ is:
\[
|S_\beta| = \sum_{\ell=1}^{t} |S^\ell_\beta| = \sum_{\ell=1}^{t} \lceil \alpha \cdot b_\ell \rceil \ge \sum_{\ell=1}^{t} \alpha \cdot b_\ell = \alpha b.
\]
Since $S_\beta$ is a subsequence of the output $S$ returned by Algorithm~\ref{alg:3}, the conditions of Lemma~\ref{lem:bound} are satisfied for $S$.
Applying Lemma~\ref{lem:bound}, with $K=\alpha b$ we obtain:
\[
f(S) \ge \left(1 - \left(1 - \frac{\beta}{b}\right)^{\alpha b} \right) f(OPT_b).
\]
Finally, using the inequality $(1-x)^k \le e^{-xk}$ for $x \in [0,1]$, we conclude that:
\[
f(S) \ge \left(1 - e^{-\alpha b} \right) \cdot f(OPT_b).
\]
    
\end{proof}

\section{Running details}
Since the most computationally intensive component of our approach is the LLM calls, which were performed via the Gemini API, all experiments were conducted on a machine with 16 GB of RAM and an Intel i7 8-core CPU. 
Experiments were implemented in Python 3.13 and repeated five times for consistency.

One challenge in our approach is ensuring that the LLM respects the minimum allocation constraints. 
While this issue was infrequent when the minimum allocation values were small, it became more pronounced as they increased. 
To address this, we allow Step 2 to attempt finding a feasible allocation multiple times when necessary.

\section{Additional experimental results}

In Figure~\ref{fig:window} we examine whether incorporating
a longer feedback history improves the performance of verbal refinement. 
The experiment details appear in Section~\ref{sec:wind} of the main paper.

\begin{figure}
    \centering
    \includegraphics[width=0.7\textwidth]{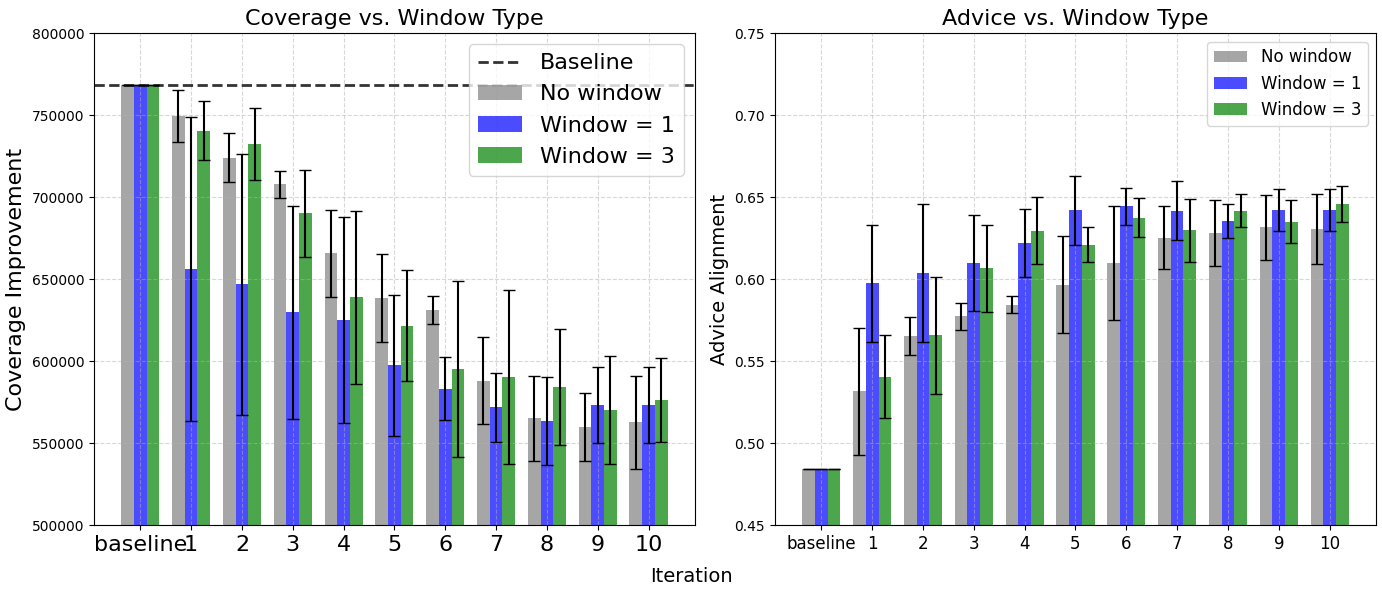}
    \includegraphics[width=0.7\textwidth]{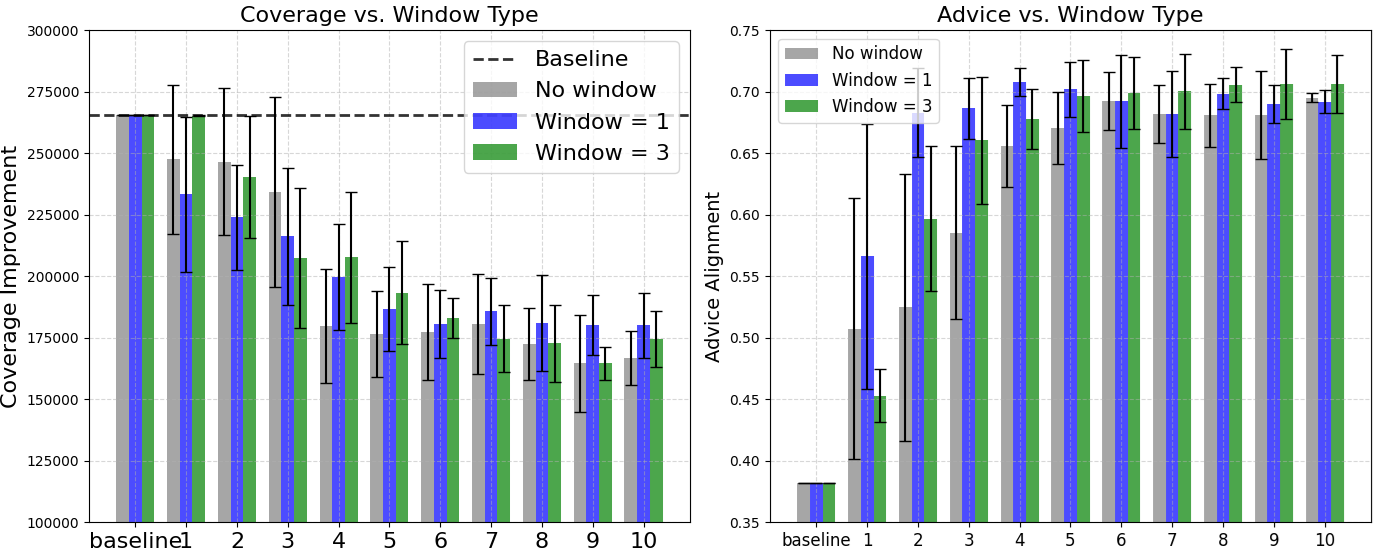}
    \caption{History length to consider in Verbal Feedback for Afar (top) and Benishangul Gumuz: 1 previous iteration vs 3 previous iterations. We considered $\alpha=0.25$ and $\beta=1.0$.}
    \label{fig:window}
    \Description{A bar chart}
\end{figure}


\end{document}